\documentclass{article}

\usepackage{graphicx}
\usepackage{amsmath, amsfonts, amssymb}
\usepackage{amsthm} 
\usepackage{mathrsfs}
\usepackage[lofdepth,lotdepth]{subfig} 
\usepackage{pbox}
\usepackage{wrapfig,lipsum,booktabs}

\newtheorem{axiom}{Axiom}
\newtheorem{myTheorem}{Theorem}

\newtheorem{lemma}{Lemma}

\DeclareSymbolFont{AMSb}{U}{msb}{m}{n}
\DeclareMathSymbol{\N}{\mathbin}{AMSb}{"4E}
\DeclareMathSymbol{\Z}{\mathbin}{AMSb}{"5A}
\DeclareMathSymbol{\R}{\mathbin}{AMSb}{"52}
\DeclareMathSymbol{\Q}{\mathbin}{AMSb}{"51}
\DeclareMathSymbol{\I}{\mathbin}{AMSb}{"49}
\DeclareMathSymbol{\C}{\mathbin}{AMSb}{"43}

\newcommand{\safety}{\mathit{safety}}
\newcommand{\saf}{\mathit{saf}}
\newcommand{\reg}{\mathit{reg}}

\begin{document}
\author{
{Brad Gulko 
\ \ \ \ \ \ \ \  Samantha Leung}\\
Computer Science Department \\
Cornell University\\
Ithaca, NY 14853 \\
bgulko $|$ samlyy@cs.cornell.edu}

\title{Maximin Safety: When Failing to Lose is Preferable to Trying to Win\thanks{We thank Joseph Y. Halpern for useful discussions and anonymous reviewers for useful comments. Work partly supported by NSF grants IIS-0812045 and CCF-1214844,  and ARO grant W911NF-09-1-0281.}}
%
%
%
%

\maketitle              

\begin{abstract}
We present a new decision rule, \emph{maximin safety}, that seeks to maintain a large margin from the worst outcome, in much the same way minimax regret seeks to minimize distance from the best.
We argue that maximin safety is valuable both descriptively and normatively. 
Descriptively, maximin safety explains the well-known \emph{decoy effect}, in which the introduction of a dominated option changes preferences among the other options.
Normatively, we provide an axiomatization that characterizes preferences induced by maximin safety, and show that maximin safety shares much of the same behavioral basis with minimax regret.
\end{abstract} 
 
\section{Introduction}

Representing uncertainty using a probability distribution, and making decisions by maximizing expected utility, is widely accepted, founded on formal mathematical principles, and satisfies intuitive notions of rationality such as independence of irrelevant alternatives and the sure thing principle \cite{Savage54}.
However, enforcing seemingly appealing concepts of rationality can ultimately lead to decisions inconsistent with what real humans consider reasonable. 
For example, observed behavior under unquantified (Knightian \cite{Knight1921}/ strict \cite{LKM10}) uncertainty, such as that in the Ellsberg paradox \cite{Ellsberg1961}, demonstrates how appealing concepts of rationality can lead to inconsistency with human choices.
Alternative decision rules, such as maximin utility \cite{Wald1950} and minimax regret \cite{Savage54,Luce1957} provide rationally plausible decisions in ambiguous situations and can be used to resolve such paradoxes, but still fail to explain some human behavioral patterns. 
A particularly illustrative example of such behavior is called the \emph{decoy effect} \cite{Huber1982}, in which the introduction of a \emph{dominated} option changes the preference among the \emph{undominated} ones. 
While the decoy effect has been investigated in the psychology \cite{Doyle1999,Wedell1991} and economics literature \cite{Bateman2008,Simonson1992}, we are unaware of any axiomatic treatment of it. 
To address this, we introduce a criterion called \emph{safety} as the basis for a \emph{maximin safety} decision rule.
	\footnote{This decision rule has been mentioned in passing, inside a proof by Hayashi \cite{Hayashi2006}, where it was referred to as `maximin joy'.
	We use the term `safety' rather than `joy' to avoid confusion with the concept called `joy of winning'	in \cite{Hayashi2006}.} 
Safety serves as a dual to regret that quantifies distance from a worst outcome, much as regret quantifies proximity to a best outcome.
Maximin safety also satisfies familiar properties common to maximin utility and minimax regret, and hence also resolves the Ellsberg paradox. 
Moreover, maximin safety accommodates observed preferences that are incompatible with minimax regret and maximin utility. 
We demonstrate how safety-seeking behavior can produce the decoy effect, and show how maximin safety can explain it.
We also extend Stoye's \cite{Stoye2011} axiomatizations of standard decision rules to include maximin safety, thus allowing a comparison between maximin safety and state-of-the-art decision rules. 

\subsection{Relative Preferences and Regret}

It is not hard to imagine situations in which performance \emph{relative} to other possible outcomes is more important than \emph{absolute} performance. 
Consider, for example, a group of duck hunters surprised by a hungry bear \cite{Crawley1995,Chao2000}. 
The hunters all attempt to escape by running in the same direction while the slowest one despairs: ``this is hopeless, we can never outrun the bear.''
The hunter in front of him snickers, ``I don't need to outrun the bear, I just need to outrun \emph{you}.''
Whether the prospect is being picked from a group of peers for a date \cite{Ariely2008}, winning a gold medal, or obtaining an `A' in a class, success is often measured by relative performance, rather than by an absolute standard. 
One such preference for relative performance is embodied in the well-known decision theoretical concept of \emph{regret} \cite{Savage54,Luce1957}.
While psychological literature on regret focuses on the bad feelings that occur \emph{after} a choice leads to an inferior outcome, some also considers that anticipation of such negative emotions may influence the choice itself \cite{Simonson1992,Larrick1995,Ritov1996}.

In this paper, we assume that uncertainty is captured by a set of possible worlds, one of which is the true state of the world. 
Regret is a measure of distance between the value of a considered outcome and the value of the best possible outcome, under a given state. 
This leads to an important property that is always true for regret -- the introduction of a \emph{dominated} option does not change the regrets of the existing options. 
We will refer to this property as \emph{independence of dominated alternatives} (IDA). 
Those who believe in regret avoidance may think that this property is perfectly reasonable. 
For example, suppose you have a \$10 bill and you can either buy a \$10 lottery ticket, or two \$5 lottery tickets.
Most would agree that your choice should not be affected by a dominated third option, ``burning the \$10 bill''.
Other standard decision rules, such as expected utility maximization, have even stronger independence guarantees. 
The ranking of two choices under expected utility maximization is \emph{menu-independent}, i.e., completely independent of the set of feasible choices (the \emph{menu}). Menu-independence implies IDA. 
In contrast, regret-based preferences are menu-dependent, but since they conform to IDA, they are not compatible with observed biases sensitive to dominated options \cite{Ariely2008}. 
While IDA seems intuitively appealing, there is a great deal of empirical evidence that human preferences are indeed affected by dominated options in measurable and sometimes profound ways.

\subsection{The Decoy Effect in Decision Theory}
Suppose you are offered \$6 in cash, and the option of trading it for a Cross pen. 
The pen is nice, but you have plenty of pens, so decide to keep the cash. 
Right before you walk away, you are offered an alternative pen in exchange for the \$6.
You see the new pen and find it hideous.
A smile comes to your face as you turn around and say, ``you know, I'll take that original Cross pen after all.''  


This story dramatizes an actual experiment  \cite{Huber1982}. 
When the first choice was offered to $106$ people, $64\%$ took the cash, $36\%$ took the pen. 
When the second pen was added to the offer to $115$ other subjects, $52\%$ took the cash, $46\%$ took the Cross pen, and $2\%$ took the decoy. 
Generally, a \emph{decoy} is an option that is designed to be inferior to another option in every way (i.e., it is a dominated option). 
Despite the intuitive appeal of IDA, the presence of a dominated option drove selection of the Cross pen from $36\%$ to $46\%$. 
In this paper, we focus on a particular class of \emph{decoy effect}, called \emph{asymmetric dominance}, which occurs when the decoy is dominated by one existing alternative, but not by another. 
Empirical studies show that the decoy is rarely chosen, but its addition to a set of choices consistently drives decision makers toward the \emph{dominating choice}.

Numerous empirical studies have also shown decoy effects in class action settlements \cite{Zimmerman2010}, recreational land management \cite{Bateman2008}, choice of healthcare plans and political candidates \cite{Hedgecock2009}, purchase of consumer goods such as cameras and personal computers \cite{Simonson1992}, restaurant choices \cite{Huber1982}, and even romantic attraction \cite{Ariely2008}. 
Surprisingly, a decoy effect can occur even if the decoy is not actually an option, but merely a recent memory of an option (a phantom decoy \cite{Farquhar1993,Doyle1999}). 
Furthermore, the decoy effect is not limited to humans, but is also observed in honeybees and grey jays \cite{Shafir2002}.

In an attempt to explain the decoy effect, experts in the behavioral sciences have offered a variety of domain-specific analyses,
including ``perceptual framing'' \cite{Huber1982}, ``value-shift'' \cite{Wedell1991}, ``extremeness aversion'' \cite{Simonson1992}, and ``contrast bias'' \cite{Simonson1992,Zimmerman2010}.
All of these explanations focus on valuing the discrepancy between the decoy and the dominating alternative. 
Intuitively, this provides a compelling example of preferring the margin of safety from the worst outcome.
As we are not aware of any formalization in decision theory that is consistent with the intuitive preference for ``margin of safety'', we offer one here.


\begin{wraptable}{r}{3cm}
\centering
		\begin{tabular}{|c|c|c|} \hline
		 & \parbox[t][0.8cm]{0.8cm}{\centering{Wet\\Road}} & \parbox[t][0.8cm]{0.8cm}{\centering Dry\\Road} \\ \hline
		 Sprint & 1 & 9 \\ \hline
		Hustle & 3 & 6 \\ \hline
		Jog & 2 & 2 \\ \hline
		\end{tabular}
	\caption{Hunters running from a bear.}
	\label{tab:hunters}
\end{wraptable}

To illustrate our new decision rule, recall the example of the unfortunate duck hunters.
As they run from the bear, they approach a blind curve and have no idea what is around it: it could be wet or dry. 
If it is dry they will cover the most ground if they try to run faster, however if it is wet (thus slippery) they will be better off if they slow down and maintain balance.
The options and the distance traveled under each circumstance are summarized in Table~\ref{tab:hunters}.
In general, exerting excessive effort on a wet road leads to slipping and less distance covered; exerting effort on a dry road leads to more distance covered.


If the probability of the road conditions is unknown, and only the first two options are available ({sprint} and {hustle}), there is no intuitively preferred choice and we may assume there are enough hunters such that at least one will pick each option. 
However, if we add a new option, {jog}, something interesting happens. 
As {jog} is dominated by {hustle}, IDA requires that its availability should not change the preferences among the other options. 
However, regardless of whether the road is wet or dry, hustle is never the worst alternative: if the road is wet, hustle $(3)$ is faster than sprint $(1)$, and if the road is dry, hustle $(6)$ is faster than jog $(2)$. 
In either case, selecting hustle prevents the hunter from being the slowest and getting caught by the bear.

While it may be callous, it seems perfectly reasonable for a hunter to decide to run just fast enough to make sure there is someone behind him. 
In other words, the most sensible decision might be to run just fast enough to guarantee the maximum possible margin between himself and the slowest runner, in the worst scenario. This margin between the hunter and his slowest compatriot can be considered a measure of $\safety$, which is at the heart of our paper.

The rest of the paper proceeds as follows. 
Section~\ref{sec:formal} provides a formalization of the decoy paradox along with basic decision-theoretical notation.  
Section~\ref{sec:safety} describes the relationship between minimax regret and \emph{maximin safety} and shows how maximin safety resolves the decoy paradox. 
Section~\ref{sec:axioms} provide an axiomatic characterization of maximin safety. 
Section ~\ref{sec:conclusion} suggests a unification of utility, regret, and safety using \emph{anchoring functions}, and also considers a generalization to qualitative relative preferences.

\section{The Formal Framework}\label{sec:formal}
Given a set $S$ of \emph{states} and a set $X$ of \emph{outcomes}, an \emph{act} $a$ (over $S$ and $X$) is a function mapping $S$ to $X$. 
The set of all acts is thus $X^S$, which we will denote by $A$.
For simplicity in this paper, we take $S$ to be finite.
Associated with each outcome $x\in X$ is a \emph{utility}: $U(x)$ is the utility of outcome $x$. 
For convenience, we will omit the explicit representation of the outcome, and denote $U(a(s))$ by $U(a,s)$ for each state $s\in S$.
We call a tuple $(S,X,U)$ a (non-probabilistic) decision problem. 
To define regret and safety, we need to assume that we are also given a set $M \subseteq A$ of feasible acts, called the \emph{menu}. 
The reason for the menu is that, as we have shown, regret and safety can depend on the menu. We will only consider finite menus, from which randomized strategies can be chosen.

\begin{wraptable}{r}{7cm}
			\begin{tabular}{|c|c|c|}\hline
			& \parbox[t][0.65cm]{1.3cm}{\centering $s_1$ \\ Safari} & \parbox[t][0.65cm]{1.8cm}{\centering $s_2$ \\World Cup} \\ \hline
			\pbox{3cm}{$a_1$: Travel} & $4$ & $4$ \\ \hline
			\pbox{3cm}{$a_2$: Sports} & $2$ & $6$ \\ \hline
			\pbox{3cm}{$a_3$: Decoy} & $3$ & $3$ \\ \hline
		\end{tabular}
	\caption{\label{tab:camera}Utilities in the camera purchase example.}
\end{wraptable}
Consider the problem of a decision maker (DM) contemplating a camera purchase, summarized in Table~\ref{tab:camera}.
The DM has a choice between buying a rugged travel camera ($a_1$) that takes decent pictures in a wide variety of circumstances, 
and buying a delicate sports camera with higher speed and image quality ($a_2$).
Each state characterizes the possible situations that a purchaser may experience during the useful life of the camera (Will the DM experience harsh conditions? Or win tickets to the World Cup?)
The utility $U(a,s)$ of act $a$ under state $s$ represents an abstract net value to the DM if the true world is state $s$.
 
If the DM ends up going on a safari $(s_1)$, then act $a_1$ results in moderate quality pictures of exciting wildlife $(U(a_1,s_1)=4)$, but act $a_2$ results in a few exquisite shots and many missed opportunities $(U(a_2,s_1)=2)$. 
On the other hand, if the DM goes to the World Cup $(s_2)$, then act $a_2$ results in many great pictures in a safe environment $(U(a_2,s_2)=6)$, while act $a_1$ provides only moderate quality pictures $(U(a_1,s_2)=4)$.

If the DM can assign probabilities $P(s_1)$ and $P(s_2)$ to the states, she can calculate an expected utility $E[U(a_i)] = \sum_{s\in S}P(s)U(a_i,s)$, and simply select the act that maximizes expected utility. 
However, if the state probabilities are unavailable, we have unquantified uncertainty. 
In such cases, the DM must find another method for aggregating the utility of each act across states in order to assign a \emph{value} to each camera. 
Here we will focus on the methods of maximax utility, minimax utility, and minimax regret. 
To understand minimax regret, we need to define the notion of regret. 
For a menu $M$ and act $a\in M$, the regret of $a$ with respect to $M$ and decision problem $(S,X,U)$ is 
\begin{align*}
\max_{s\in S}( \max_{a'\in M}U(a',s) - U(a,s) ).
\end{align*}
We denote this as $\reg_M^{(S,X,U)}(a)$, and usually omit the superscript $(S,X,U)$.

When comparing decision rules, it is often convenient to define a \emph{value function} that assigns a numeric value to each act, for the purpose of ranking the acts.
Formally, for a decision problem $(S,X,U)$, a value function is a function 
$$V^{(S,X,U)}(a,M): X^S \times 2^{A} \rightarrow \R.$$
We will usually omit the superscript $(S,X,U)$ and just write $V(a,M)$, or $V(a)$ if the value function is menu-independent.

We say that the value function $V$ \emph{represents} the family of preference relations $\succ_{V,M}$, if for all menus $M$ and all $a,a'\in M$,
$$a \succ_{V,M} a' \Leftrightarrow V(a,M) > V(a',M).$$
In other words, act $a$ is (strictly) preferred to act $a'$ with respect to menu $M$ if and only if $V(a,M) > V(a',M)$.
The value functions and preferences of several standard decision rules are given in Table~\ref{tab:decisionrules}.

\begin{table*}[tb]
	\centering
		\begin{tabular}{|c|c|p{5.6cm}|c|}  \hline
		Decision Rule & Value of an act $a$ & \centering Decision rule description & Best \\ \hline
		maximax utility & $V(a) = \max_{s\in S} U(a,s)$ & Optimize the best-case outcome. & $a_2$ \\ \hline
		maximin utility & $V(a) = \min_{s\in S} U(a,s)$ & Optimize the worst-case outcome. & $a_1$ \\ \hline
minimax regret & $V(a,M) = -\reg_M(a)$ & Pick an act to minimize the worst-case distance from the best outcome. & $a_1,a_2$ \\ \hline
		\end{tabular}
	\caption{\label{tab:decisionrules}Standard decision rules and most valued acts in the camera example.}
	\vspace{-15pt}
\end{table*}

Now, perhaps the camera vendor would like to sell more travel cameras, so the vender puts an obsolete travel camera $a_3$ next to $a_1$ as a \emph{decoy}.
Camera $a_3$ has the same price as $a_1$, but fewer features and lower picture quality.
The vendor hopes to make $a_1$ more appealing by contrast with $a_3$.
Table~\ref{tab:camera} illustrates the decision problem when $a_3$ is added to the menu.
The ranking between $a_1$ and $a_2$ according to each of the decision rules in Table~\ref{tab:decisionrules} is unaffected by the introduction of $a_3$ to the menu.
The addition of $a_3$ also illustrates the concept of dominance. 
We say that an act $a$ \emph{dominates} $a'$, if for all $s\in S$, $U(a,s) > U(a',s)$. 


\section{Maximin Safety} \label{sec:safety}

While minimax regret seeks to minimize separation from best outcomes, \emph{maximin safety} is a conceptual dual that seeks to maximize separation from the worst outcomes.
For a menu $M$ and act $a\in M$, the safety of $a$ in state $s$ is defined as:
$$ \safety^{(S,X,U)}_M(a,s) = U(a,s) - \min_{a'\in M}( U(a',s) ), $$
and in keeping with the convention for regret, the safety of an act is defined as:
$$
\safety^{(S,X,U)}_M(a) = \min_{s\in S}( \safety^{(S,X,U)}_M(a,s) ).
$$
We will often omit the superscript $(S,X,U)$. 

The family of maximin safety preferences $\succ_{\saf,M}$ represented by the $\safety$ value function satisfies, for all $M$ and $a,a'\in M$,
$$a \succ_{\saf,M} a' \Leftrightarrow \safety_M(a) > \safety_M(a').$$

\begin{table}[htbp]
\centering
	\centering
		\begin{tabular}{|c|c|c|c|c|c|c|} \hline
			& \multicolumn{2}{c}{\parbox[c][0.7cm]{0.95cm}{Utility}} & \multicolumn{2}{|c|}{\parbox[t][1.5cm]{1.55cm}{\centering Safety \\ (no decoy)}} & \multicolumn{2}{|c|}{\parbox[t][0.7cm]{1.55cm}{\centering Safety \\(w. decoy)}} \\ \hline
			& $s_1$ & $s_2$ & $s_1$ & $s_2$ & $s_1$ & $s_2$  \\ \hline
			$a_1$: travel & 4 &4 & 2 & 0 & 2 & 1  \\ \hline
			$a_2$: sports & 2 & 6 & 0 & 2 & 0 & 3  \\ \hline
			$a_3$: decoy & 3 & 3 & \multicolumn{2}{|c|}{} & 1 & 0  \\ \hline
		\end{tabular}
	\caption{\label{tab:table5}Camera purchase {with} and without decoy.}
\end{table}
\begin{table}
	\centering
		\begin{tabular}{|c|c|c|c|c|c|c|c|c|} \hline
			& \multicolumn{2}{|c|}{\parbox[c][0.5cm]{0.9cm}{\centering Utility}} & \multicolumn{2}{|c|}{\parbox[c][0.5cm]{0.9cm}{\centering Regret}} & \multicolumn{2}{|c|}{\parbox[c][0.5cm]{0.9cm}{\centering Safety}} & Optimal for\\ \hline
			& $s_1$ & $s_2$ & $s_1$ & $s_2$  & $s_1$ & $s_2$ & \\ \hline
			$a_1$  & 1 & 9 & 3 & 0  & 0 & 5 &  maximax utility\\ \hline 
			$a_2$  & 3 & 6 & 1 & 3  & 2 & 2 &  maximin safety \\ \hline
			$a_3$  & 2 & 7 & 2 & 2  & 1 & 3 &  minimax regret \\ \hline
			$a_4$  & 4 & 4 & 0 & 5  & 3 & 0 &  maximin utility\\ \hline
		\end{tabular}
	\caption{\label{tab:table6}Different decision rules select different acts for the same problem.}
\end{table}
Now we reconsider the camera example using safety (Table~\ref{tab:table5}).
Without the decoy, both acts have the same safety of 0, since each act has the lowest utility in some state; so there is no clear safety preference. 
However, when the decoy is present, the act $a_1$ never has the lowest utility at any state, and thus it has a strictly positive safety. 
In this case, $\safety_{\{a_1,a_2,a_3\}}(a_1)$ is the unique maximum among the acts $\{a_1,a_2,a_3\}$, and therefore $a_1$ is the preferred choice.
The relative increase in the safety of an act due to the addition of the dominated act is an essential element in solving the decoy paradox. 
Intuitively, this may correspond to a sense that even if a particular act gets low utility in the realized state, the DM may think that ``I'm better off than the fools who bought the worse camera'', or in a more positive light, ``I must be getting a steal with this better camera for the same price''. 
In competitive survival games (such as the reality game show Survivor), the notion of maximizing safety may also embody a preference to maintain a maximal distance from the lowest performer, which reduces the chance of elimination.
Table~\ref{tab:table6} compactly demonstrates how choices based on maximin safety differs from the other standard decision rules.

\begin{wraptable}{r}{8cm}
\centering
		\begin{tabular}{|c|c|c|c|c|c|c|c|c|c|} \hline
			& \multicolumn{3}{|c|}{Utility}  & \multicolumn{3}{|c|}{\parbox[c][0.95cm]{1.6cm}{\centering Safety(no\\decoy)}} & \multicolumn{3}{|c|}{\parbox[c][0.95cm]{1.6cm}{\centering Safety (w. \\ decoy $a4$)} } \\ \hline
			& $s_1$ & $s_2$ & $s_3$  & $s_1$ & $s_2$ & $s_3$ & $s_1$ & $s_2$ & $s_3$ \\ \hline
			$a_1$  & 9 & 2   & 6 & 5 & 0 & 0 	& 8 & 0 & 0 \\ \hline
			$a_2$  & 5 & 3 & 7 & 1 & 1 &1  		& 4 & 1 &1  \\ \hline
			$a_3$  & 4 & 8 & 8 & 0 & 6 & 2 		& 3 & 6 & 2    \\ \hline
			$a_4$  & 1 & 5 & 6 &  \multicolumn{3}{|c|}{}  		& 0 & 3 & 0  \\ \hline
		\end{tabular}
	\caption{\label{tab:table7}Without $a_4$, $M = \{a_1,a_2,a_3\}$, and $a_2 \succ_{\saf,M} a_3$. Adding $a_4$ (dominated by $a_3$) reverses the maximin safety preference between $a_2$ and $a_3$.}
\end{wraptable}
In the camera example, the addition of a decoy created a strict preference between two acts that were initially tied. 
The introduction of a dominated act can actually \emph{reverse} preferences between acts.
Table~\ref{tab:table7} shows a menu of three acts: $M=\{a_1,a_2,a_3\}$.
Act $a_2$ has a minimum safety of $1$, while both $a_1$ and $a_3$ have the lowest utility for some state, so each has minimum safety of $0$.
Consequently, $a_2$ is the most preferred choice under the safety preference.
When a new choice $a_4$ is added, act $a_4$ is dominated by $a_3$, but it has higher utility than the other acts in some states. 
This situation is known as \emph{asymmetric dominance}, which is typically associated with decoy effects. 
In this example, asymmetric dominance guarantees that $a_3$ is never one of the worst choices, and thus has a strictly positive safety value. 
In other words, the addition of $a_4$ to the menu $M$ does not affect the safety of $a_1$ or $a_2$, but increases the safety of $a_3$ to make $a_3 \succ_{\saf,M\cup\{a_4\}} a_2$.

\section{Axiomatic Analysis}\label{sec:axioms}

To provide an axiomatic characterization of maximin safety, we employ the standard \emph{Anscombe-Aumann} (AA) framework \cite{AnscombeAumann1963}, where outcomes are restricted to lotteries. 
Maximin safety is characterized by modifying one of the axioms in an existing characterization of minimax regret provided by Stoye \cite{Stoye2011}.

Given a set $Y$ of \emph{prizes}, a \emph{lottery} over $Y$ is just a probability with finite support on $Y$.  
As in the AA framework, we let the set of \emph{outcomes} be $\Delta(Y)$, the set of all lotteries over $Y$.  
Thus, \emph{acts} are functions from $S$ to $\Delta(Y)$.
We can think of a lottery as modeling objective, quantified uncertainty, while the states model unquantified uncertainty. 
The technical advantage of considering such a set of outcomes is that we can consider convex combinations of acts.  
If $f$ and $g$ are acts, define the act $\alpha f + (1-\alpha)g$ to be the act that maps a state $s$ to the lottery $\alpha f(s) + (1-\alpha)g(s)$.  
For simplicity, we follow Stoye \cite{Stoye2011} and restrict to menus that are the convex hull of a finite number of acts, so that if $f$ and $g$ are acts in $M$, then so is $pf + (1-p)g$ for all $p \in [0,1]$.

In this setting, we assume that there is a utility function $U$ on prizes in $Y$, and that there are at least two prizes $y_1$ and $y_2$ in $Y$, with different utilities.
Note that $l(y)$ is the probability of getting prize $y$ if lottery $l$ is played. 
We will use $l^*$ to denote a constant act that maps all states to $l$.
The utility of a lottery $l$ is just the expected utility of the prizes obtained, that is, $u(l) = \sum_{\{y \in Y \colon l(y)>0\}} l(y) U(y).$
The expected utility of an act $f$ with respect to a probability $\Pr$ is then just $u(f) = \sum_{s \in S} \Pr(s) u(f(s))$, as usual.  
Given a set $Y$ of prizes, a utility $U$ on the prizes, and a state space $S$, we have a family $\succeq^{S,\Delta(Y),u}_{\saf,M}$ of preference orders on acts determined by maximin safety, where $u$ is the utility function on lotteries as determined by $U$.\footnote{
We let $f \succeq^{S,\Delta(Y),u}_{\saf,M} g$ iff $g \not \succ^{S,\Delta(Y),u}_{\saf,M} f$, and $f \sim_M g$ iff $f \succeq_M g$ and $g \succeq_M f$.}
For convenience, from here on we will write $\succeq^{S,Y,U}_M$ rather than $\succeq^{S,\Delta(Y),u}_{\saf,M}$.
We will state the axioms in a way such that they can be compared to standard axioms and those for minimax regret in \cite{Stoye2011}.
The axioms are universally quantified over acts $f$, $g$, and $h$,  menus $M$ and $M'$, and $p \in (0,1)$.  
Whenever we write $f\succeq_M g$ we assume that $f,g\in M$.

\begin{axiom}{(Monotonicity)}\label{axiom:m}
  $f\succeq_M g$ if $(f(s))^* \succeq_{\{(f(s))^*,(g(s))^*\}} (g(s))^*, \forall s\in S$ .
\end{axiom}
\begin{axiom}{(Completeness)} $f\succeq_M g$ or $g\succeq_M f$.
\end{axiom}
\begin{axiom}{(Nontriviality)}\label{axiom:nt}
$f\succ_M g$ for some acts $f$ and $g$ and menu $M$.
\end{axiom}
\begin{axiom}{(Mixture Continuity)}
If $f\succ_M g\succ_M h$, then there exists 
$ q,r\in (0,1)$ such that $q f + (1-q) h \succ_M g \succ_M r f + (1-r) h.$
\end{axiom}
\begin{axiom}{(Transitivity)}\label{axiom:t} $f\succeq_M g\succeq_M h\Rightarrow f\succeq_M h$.
\end{axiom}
Menu-independent versions of Axioms~\ref{axiom:m} to \ref{axiom:t} are standard in other axiomatizations, and in particular hold for maximin utility.
Axiom \ref{axiom:nt} is used in the standard axiomatizations to get a nonconstant utility function in the representation.
While maximin safety does not satisfy menu-independence, it does satisfy menu-independence when restricted to menus consisting of only constant acts.
This property is captured by the following axiom.

\begin{axiom}{(Menu independence for constant acts)} If $l^*$ and $(l')^*$ are
constant acts, then $l^* \succeq_M (l')^*$ iff $l^* \succeq_{M'} (l')^*$.
\end{axiom}

We also have a menu-dependent version of the von Neumann-Morgenstern (VNM) Independence axiom.
Like the VNM Independence axiom, Axiom~\ref{axiom:IND} says that ranking between two acts does not change when both acts are mixed with a third act; but unlike VNM Independence, the menu used to compare the original acts in Axiom~\ref{axiom:IND} is different from that used to compare the mixtures.
Axiom~\ref{axiom:IND} holds for minimax regret and maximin safety, but not for maximin utility.

\begin{axiom}{(Independence)}\label{axiom:IND}
$f\succeq_M g \Leftrightarrow pf + (1-p) h \succeq_{pM + (1-p)h} pg + (1-p) h$.
\end{axiom}

\begin{axiom}{(Symmetry)}
For a menu $M$, suppose that $E,F\in 2^S \backslash\{\emptyset\}$ are disjoint events such that 
for all $f\in M$, $f$ is constant on $E$ and on $F$. Define $f'$ by 
$$
\begin{array}{ll}
f'(s) = \begin{cases}
f(s')\text{ for some }{s'\in E}, \text{ if } s\in F \\
f(s')\text{ for some }{s'\in F}, \text{ if } s\in E \\
f(s) \text{ otherwise }
\end{cases}
\end{array}$$
Let $M'$ be the menu generated by replacing every act $f\in M$ with $f'$. 
Then 
$$
\begin{array}{ll}
f\succeq_M g \Leftrightarrow f'\succeq_{M'} g'
\end{array}$$
\end{axiom}
Symmetry, which is one of the characterizing axioms for minimax regret in \cite{Stoye2011}, captures the intuition that no state can be considered more or less likely than another. 
Therefore Symmetry helps distinguish the probability-free decision rules maximin utility, minimax regret, and maximin safety, from their probabilistic counterparts \cite{GilboaSchmeidler1989,StoyeRegret}.
\begin{axiom}{(Ambiguity Aversion)}\label{axiom:AA}
$f\sim_M g \Rightarrow pf + (1-p)g \succeq_M g.$
\end{axiom}
 Axiom \ref{axiom:AA} says that the decision maker weakly prefers to hedge her bets. 
Axioms 1-\ref{axiom:AA} are all part of the characterization in \cite{Stoye2011} of minimax regret (which consists of Axioms 1-\ref{axiom:AA} and Symmetry).
Axioms 1-5 and \ref{axiom:AA} are also sound for the maximin decision rule \cite{Stoye2011}. 

In \cite{Stoye2011}, one of the axioms characterizing minimax regret is Independence of Never Strictly Optimal alternatives (INA), which states that adding or removing acts that are not strictly potentially {optimal} in the menu does not affect the ordering of acts. \footnote{
An act $h$ is \emph{never strictly optimal  relative to $M$} if, for all states $s \in S$, there is some $f \in M$ such that $(f(s))^* \succeq (h(s))^*$. }
By varying this INA axiom, we obtain a characterization for maximin safety. 
We say that an act $a$ is \emph{never strictly worst relative to $M$} if, for all states $s\in S$, there is some $a'\in M $ such that $a(s) \succeq a'(s)$.
\begin{axiom}{(Independence of Never Strictly Worst Alternatives (INWA))}\label{axiom:INWA}
If an act $a$ is never strictly worst relative to $M$, then $f \succeq_M g$ \emph{iff} $f \succeq_{M\cup\{a\}} g$.
\end{axiom}

Although adding acts to the menu, in general, can affect minimax regret preferences, INA implies the Independence of Dominated Alternatives property that we used earlier when discussing the decoy effect.
Thus, INA guarantees that minimax regret can never be compatible with the decoy effect.

\begin{myTheorem}\label{thm:safety}
For all $Y,U,S,$ the family of maximin safety preference orders $\succeq^{S,Y,U}_{\saf,M}$ induced by a decision problem $(S,\Delta(Y),u)$ 
satisfies Axioms \ref{axiom:m}--\ref{axiom:INWA}.  
Conversely, if the family of preference orders $\succeq_{M}$ on the acts in
$\Delta(Y)^S$ satisfies Axioms \ref{axiom:m}--\ref{axiom:INWA}, then there exists a  
utility function $U$ on $Y$ that determines a utility $u$ on $\Delta(Y)$ such that
$\succeq_M = \succeq_{\saf,M}^{S,Y,u}$.
Moreover, $U$ is unique up to affine transformations.
\end{myTheorem}
\begin{proof}
The soundness of the axioms are straightforwardly verified, so we show only the completeness of the axioms.
We will use the same general sequence of arguments that Stoye uses in \cite{Stoye2011}.
First, we establish a nonconstant utility function $U$, where constant acts are ranked by their expected utilities. 
Since we have the standard axioms ($1-5$), we get $U$ from standard arguments, and it is unique up to affine transformations. 
Next, we observe the following lemma:
\begin{lemma}\label{lem:one}Suppose the family $\succeq_M$ satisfies Axioms 1-10, and $\succeq_{M^+}$ is representable by maximin safety, where $M^+$ is the menu of all acts with nonnegative utilities. Then the family $\succeq_M$ is representable by maximin safety.
\end{lemma}
Lemma~\ref{lem:one} follows from an argument analogous to that for regret in \cite{Stoye2011}.
The next step is to establish that the axioms on $\succeq_M$ restrict $\succeq_{M^+}$ to satisfy the axioms of ambiguity aversion, monotonicity, completeness, transitivity, non-triviality, and symmetry. It is a straightforward verification that will not be reproduced here. 
Theorem 1 (iii) of \cite{Stoye2011} then implies that $\succeq_{M^+}$ is the maximin utility ordering. 
Next, let $g_M$ be an act such that $u\circ g_M(s) = - \min_{h\in M}u(h,s)$, so that we have
\newline
$$
\begin{array}{ll}
f \succeq_M g  & \Leftrightarrow  \frac{1}{2}f + \frac{1}{2} g_M \succeq_{M^+} \frac{1}{2}g + \frac{1}{2}g_M \\
 & \Leftrightarrow  \min_{s\in S}u(\frac{1}{2}f + \frac{1}{2} g_M ,s) \geq \min_{s\in S}u(\frac{1}{2}g + \frac{1}{2}g_M,s)\\
 & \Leftrightarrow  \displaystyle{\min_{s\in S}}( \frac{1}{2}( u(f,s)- \displaystyle{\min_{h\in M}}u(h,s)) ) \geq 
			\displaystyle{\min_{s\in S}}( \frac{1}{2}( u(g,s)-\displaystyle{\min_{h\in M}}u(h,s)) ).
\end{array}
$$
\qed
\end{proof}

\vspace{-9pt}
The characterizing axioms serve as a justification for maximin safety in the sense that behaving as a safety maximizer is equivalent to accepting the axioms. 
Axioms 1-7 are standard and broadly accepted to be reasonable, while symmetry and ambiguity aversion are implied by both maximin utility and minimax regret.
Whether the INA axiom (for regret) or the INWA axiom (for safety) is more reasonable would depend on the individual and the nature of the decision problem.
Thus, we believe that the reasonableness of the maximin safety decision rule is comparable to that of minimax regret.

Individual necessity of the axioms can be established, as is commonly done \cite{Hayashi2006,Stoye2011}, by giving examples of preferences that satisfy all the axioms except for the one whose necessity is being shown. 
For the axioms shared with minimax regret, the same examples found in \cite{Stoye2011} shows their individual necessity. 
For the INWA axiom, the required example is minimax regret. 
Indeed, a decision rule equivalent to maximin safety was used by Hayashi \cite{Hayashi2006} as an example to justify minimax regret's entailment of INA.

Clearly, just as minimax regret is readily generalized to \emph{minimax expected regret} when uncertainty is represented by a set of probability distributions over the state space, maximin safety can be readily extended to \emph{maximin expected safety} in the same manner. 
As one would expect, given an axiomatization of minimax expected regret \cite{StoyeRegret}, the modification of the INA axiom to INWA results in an axiomatization for maximin \emph{expected} safety. 

\section{Discussion, Generalizations, and Future Work}\label{sec:conclusion}

Both minimax regret and maximin safety embody preferences based on \emph{relative}, rather than \emph{absolute} utility. 
In Table~\ref{tab:table6}, the act preferred by safety has a lower minimum utility than the act preferred by maximin utility, just as the act picked  by minimax regret neglects a higher maximum utility in order to minimize the \emph{margin} to each state's maximum utility.
The shared preference for relative over absolute performance is reflected in a striking similarity in the structure of the value functions for regret and safety. In comparison, minimax regret can be expressed for all acts $a,b$ as: 
$$
\begin{array}{ll}
a \succ_{\reg,M} b \text{ iff } &\displaystyle{\min_{s\in S}}(U(a,s) - \displaystyle{\max_{a'\in M}}U(a',s)) > \displaystyle{\min_{s\in S}}(U(b,s) - \displaystyle{\max_{a'\in M}}U(a',s)).
\end{array}
$$
Similarly, maximin safety is represented for all acts $a,b$ as 
$$
\begin{array}{ll}
a \succ_{\saf,M} b \text{ iff } &\displaystyle{\min_{s\in S}}(U(a,s) - \displaystyle{\min_{a'\in M}}U(a',s)) > \displaystyle{\min_{s\in S}}(U(b,s) - \displaystyle{\min_{a'\in M}}U(a',s)).
\end{array}
$$
The structural resemblance suggests a common form for the value function. By defining a menu-dependent \emph{anchoring function} $t: S \times 2^A \rightarrow \R$, we can represent several previously discussed value functions as:
$$
\begin{array}{ll}
V_t(a,M) = \min_{s\in S} U'(a,s,M,t),
\end{array}
$$
where $U'(a,s,M,t) = U(a,s) - t(s,M)$ can be viewed as an \emph{anchored effective utility}. 
One can see that $V_t$ represents maximin utility if $t(s,M)=0$; minimax regret if $t(s,M)=\max_{a'\in M}U(a',s)$; and maximin safety if $t(s,M) = \min_{a'\in M}U(a',s)$.
Note that by varying just the anchoring function $t$, we can obtain all the mentioned decision rules, and more.
While we focus only on maximin safety in this paper, other forms for $t(s,M)$ maximize the positive margin from a state-dependent average, median, or some other characteristic of interest to a DM. 
For example, college students might seek to conservatively maximize their margin above a desired quantile, in order to achieve a particular grade.


The present work is motivated by behavioral observations of the decoy effect that are typically described in empirical quantities such as distance, price and volume, and thus is most intuitive in a \emph{quantitative} framework. 
However, the key observation is that safety, like regret, is a notion of relative performance with respect to a set of outcomes, rather than absolute performance. 
As absolute quantitative utility $U:X \to \R$ can be generalized to a qualitative framework by replacing the $U$ with a mapping $X \to L$ for some ordered set $L$, relative utility may be made qualitative by considering the mapping with $2^X \times X \to L$. 
In the case of safety and regret, the particular element of $2^X$ is the set of all possible outcomes in a state, given a menu of acts. 
Aggregation of $N$ state-specific orderings into an ordering over acts can be accomplished by an aggregation function $M:L^N \to L$ \cite{Marichal2001}. 
This generalization can be readily applied to various characterizations of uncertainty, including probability, plausibility, and the strict uncertainty used in this paper \cite{LKM10}. 
While the authors expect that the present quantitative axiomatization can be adapted to a qualitative framework (see, e.g. \cite{DFP03}), it is beyond the scope of the current paper. 
 
\bibliographystyle{abbrv}
\bibliography{awareness,safety}

\end{document}